\documentclass[11pt]{article}
\usepackage{fullpage}
\usepackage{booktabs} 

\usepackage{amsmath,amssymb,amsfonts}

\usepackage{amsthm}
\usepackage{natbib}
\usepackage{hyperref}
\usepackage[T1]{fontenc}
\usepackage{xcolor}
\usepackage{algorithm}
\usepackage{algpseudocode}
\usepackage{graphicx}

\usepackage{tikz}
\usepackage{tkz-euclide}    
\usepackage{pgfplots}       
\usepackage{subfigure}
\pgfdeclarelayer{background}
\pgfsetlayers{background,main}

\newcommand{\field}[1]{\mathbb{#1}}
\newcommand{\R}{\field{R}} 
\DeclareMathOperator*{\argmax}{arg\,max}

\newtheorem{theorem}{Theorem}[section]
\newtheorem{proposition}[theorem]{Proposition}
\newtheorem{corollary}[theorem]{Corollary}
\newtheorem{lemma}[theorem]{Lemma}
\newtheorem{assumption}{Assumption}

\theoremstyle{definition}
\newtheorem{example}[theorem]{Example}
\newtheorem{definition}[theorem]{Definition}

\title{Decision-aid or Controller? Steering Human Decision Makers with Algorithms}

\author{Ruqing Xu$^1$, Sarah Dean$^2$}
\date{%
$^1$Cornell University, Department of Economics\\%
$^2$Cornell University, Department of Computer Science\\[2ex]%
January, 2023
}

\begin{document}

\maketitle

\begin{abstract}
Algorithms are used to aid human decision makers by making predictions and recommending decisions. Currently, these algorithms are trained to optimize prediction accuracy. What if they were optimized to control final decisions? In this paper, we study a decision-aid algorithm that learns about the human decision maker and provides ``personalized recommendations’’ to influence final decisions. We first consider fixed human decision functions which map observable features and the algorithm’s recommendations to final decisions. We characterize the conditions under which perfect control over final decisions is attainable. Under fairly general assumptions, the parameters of the human decision function can be identified from past interactions between the algorithm and the human decision maker, even when the algorithm was constrained to make truthful recommendations. We then consider a decision maker who is aware of the algorithm’s manipulation and responds strategically. By posing the setting as a variation of the cheap talk game \citep{crawford1982strategic}, we show that all equilibria are partition equilibria where only coarse information is shared: the algorithm recommends an interval containing the ideal decision. We discuss the potential applications of such algorithms and their social implications. 
\end{abstract}

\section{Introduction}

A large and growing body of empirical literature has identified inconsistencies, errors, and biases in humans when making consequential decisions \citep{kleinberg2018human, ludwig2021fragile}. \citet{mullainathan2022diagnosing} show that compared to an algorithmic model predicting heart attack risk, physicians simultaneously overtest low-risk patients and undertest high-risk patients, indicating inaccurate prediction in diagnosing. Under reasonable behavioral and econometric assumptions, \citet{rambachan2021identifying} is able to identify systematic prediction mistakes about defendants' failure-to-appear risk in at least $20\%$ of New York City judges. Judges' sentencing decisions are even influenced by irrelevant factors like recent sports-team losses \citep{eren2018emotional} or hot weather \citep{heyes2019temperature}. 

Toward the goal of helping humans make better decisions, decision-aid algorithms have been developed and employed in many consequential domains, including healthcare, criminal justice, and credit lending systems. Much technical literature in this area has focused on improving accuracy and ensuring fairness of recommended decisions. 
While this literature has made fruitful progress, good recommendations do not directly translate to good final decisions. There still exists a gap between the algorithm's recommendation and the human's ultimate decision.
As illustrated in Figure~\ref{fig:diagram}, the decision depends on how the human agent responds to the recommendation. 

\begin{figure}[b]
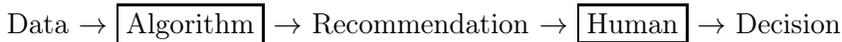

    \centering
    Data $\rightarrow$ \fbox{Algorithm} $\rightarrow$ Recommendation $\rightarrow$ \fbox{Human} $\rightarrow$ Decision
    \caption{System diagram of algorithm-aided decision process. While the algorithm provides a recommendation, the human makes the final decision.}
    \label{fig:diagram}
\end{figure}

Indeed, recent research on evaluating the outcome of such decision-aid systems has shown dispiriting results. In the medical domain, \citet{jacobs2021machine} show that receiving machine-learning recommendations did not improve clinicians' accuracy of selecting antidepressants compared to the control group in which clinicians made treatment decisions independently. Moreover, the clinicians with or without machine-learning aids both performed significantly worse than the machine-learning system alone. Similarly, in the legal domain, \citet{imai2020experimental} and \citet{stevenson2022algorithmic} study the impacts of adopting algorithmic risk assessment tools in pretrial detention and sentencing decisions respectively. Both papers find that risk assessment tools had little overall impact on the judges' final decisions, leading to no significant gain in public safety or reduced incarceration rates. 

In addition to empirical research suggesting the importance of addressing the gap between recommendations and decisions, theoretical literature on fair machine learning also recognizes the role of human intermediaries in the decision making pipeline \citep{suresh2021framework}. However, as far as we know, studies of formal models for human intermediaries in the decision pipeline are lacking. 

We present a model of the interaction between decision-aid algorithms and human intermediaries. We study a specific question: How much can an algorithm influence human decisions by recommendations alone, and what properties of human decision makers allow for such influence?
Toward that end, we first consider the case of a fixed human decision function mapping observable features and an algorithm’s recommendation to a final decision. When such a decision function is known, we characterize the conditions under which perfect control over human decisions is attainable. When the function is unknown, we develop general conditions under which it can be learned from past interactions between the algorithm and the decision maker. In the final section, we introduce a decision maker who becomes aware of algorithm's manipulations and responds strategically. We show that in equilibrium only coarse information is shared: the algorithm recommends only an interval containing the ideal decision. The decision maker makes a utility-maximizing decision on the basis of this coarse recommendation. 

Before formally introducing our model, we discuss some related work on algorithm-aided decision making.
Most closely related is the literature on human-machine complementarity. Although humans may exhibit inconsistency, error, or discrimination in decision making, they may also observe private information that is not observed by the algorithm and thus have a comparative advantage in those cases \citep{yeomans2019making, de2020case}. Therefore, this line of research has focused on designing optimal rules to combine the independent decisions of the algorithm and the judge into a joint decision. For example, \citet{RLHH22} analyzes human-machine complementarity in settings where an independent third party combines human and machine predictions. \citet{10.1145/3531146.3533221} give conditions on when complementarity is impossible and when complementarity is achievable.  

Our work differs from the complementarity literature in two significant ways. First, we focus on settings in which the human still remains the final decision maker in the system. Unlike in hybrid systems where the algorithm's prediction directly feeds into the final decision, the algorithm in our setting merely acts as an ``advisor'' and has sway over final decisions only through its recommendations. We note that this setting is arguably more common in real life, for a variety of reasons including the preservation of human agency and accountability. 
Second, our work also differs in the nature of the problem we try to solve. The complementarity literature focuses on the problem of utilizing different information observed by the algorithm and the human, while implicitly assuming the same preference of the two parties. In contrast, this paper focuses on the setting where the decision maker has fundamentally different objectives compared to the algorithm, while each party seeks to maximize their own utility dependent on the final decision.

For example, a loan manager may exhibit biases against certain racial or sexuality groups that are not aligned with the algorithm's objective of maximizing firm profits. A hiring manager may disregard affirmative action mandates by the firm in his actual practices. In both cases, a principal-agent problem arises. An algorithm that learns about the decision maker may be able to exercise the principal's objective even after going through an agent. 
In some sense, this paper investigates a ``soft'' way to implement the algorithm's decision,
approximating the replacement of a human decision maker with an algorithm.
In Section~\ref{sec:discussion}, we discuss the possible benefits and harms of such an approach.

\subsection{Vignettes}
\begin{example}[Criminal justice] In the setting of criminal sentencing, the algorithm is tasked with recommending a sentence length to the judge based on the features of the case and the defendant. The judge observes and considers some features and may or may not use the recommendation from the algorithm in making a final sentencing decision. 
\end{example}

\begin{example}[Credit lending] The algorithm takes in the features of the loan applicant (e.g., credit history, income, demographic information) and recommends a credit score to the loan manager. The loan manager then decides to accept or reject the application, and if accepted, what interest rate to charge. Suppose that the algorithm learns from the loan manager's past decisions that he is unjustifiably lenient to applicants of his own racial group. If the algorithm recommends a lower credit score for applicants of this racial group, this may offset the loan manager's preferential bias.
\end{example}

\begin{example}[Cognitive bias]
Psychological and behavioral studies have documented that humans do not usually treat probabilities linearly. A loan manager may overweigh a small probability of default. As a result, he could make overly conservative decisions systematically near the lower end of the default risk distribution. In this case, an algorithm that learns about the loan manager can adapt its recommendations to help correct this cognitive bias---for example, appropriately attenuating the default risk toward zero. 
\end{example}

\section{Problem Setting}
Throughout this paper, we will consider an ``algorithm'' that gives a recommendation about an individual to a ``judge,'' who then makes a final decision (see Figure~\ref{fig:diagram}). We use the word ``judge'' as a general name for the human decision maker who passes the final judgement, not necessarily indicating a judicial setting.
We fix one algorithm and one judge so that we may omit judge-specific subscripts.

We consider the setting where the variables of interest are continuous. The algorithm is tasked with predicting and recommending a continuous value to the judge, denoted by $n$. The judge also makes a continuous decision denoted by $y$. 
The ``ideal'' decision is also continuous and denoted by $m$.
Without loss of generality, we normalize and scale the  ideal decision, the algorithm's recommendation, and the judge's final decision to be on the same scale within the unit interval $[0,1]$.
For example, such a normalization and rescaling in the lending scenario might correspond to mapping credit scores and interest rates to repayment probabilities.
Throughout the paper, we will use the convention that higher decision values have a positive connotation.

For an individual decision subject $i$,
the algorithm makes use of high-dimensional features $x_i\in\mathcal X\subset \R^D$.
We assume that
the judge can only process (or access) a subset of these features. This gives the judge an incentive to use the algorithm, even if their objectives are not necessarily aligned. 
We define an ``attention function'' $A: \mathcal{X} \to \{0,1\}^D \odot \mathcal{X}$, where $\odot$ denotes the entry-wise multiplication of vectors.  The attention function picks out the entries of $x_i$ that impact the judge's decisions. 
We use $x_i' = A(x_i)$ to denote these features for an individual $i$, and $D'$ to denote the reduced dimension of these features.
We further suppose that the algorithm is aware of which features impact the judge's decision, i.e., that the algorithm knows $A$.

As a starting point, we suppose that the ideal decision for individual $i$ is determined by their feature so that $m_i=f(x_i)$ for a fixed function $f$.
We further suppose that the algorithm knows the function $f$, e.g., due to ample training data.
While a ``truthful'' algorithm would faithfully recommend this ideal decision to the judge,
we will consider a general recommendation 
$n_i$ which may not necessarily equate the ideal decision when the algorithm takes into account the judge's reactions.
Finally, we denote the judge's decision rule as
$y_i = g(n_i, x_i')$, which can depend on the  recommendation and the observed features.

The objective of the algorithm is to select a recommendation rule which minimizes the difference between the ideal decision and judge's final decision. 
In the first part of this paper, we investigate the design of a recommendation rule when the judges decision rule is static. 
In Section~\ref{sec:strategic}, we consider the game that arises when the judge is aware of potential manipulation by the algorithm and can change his decision rule accordingly.

 \begin{table}[h]
     \centering
     \begin{tabular}{c|ccc|c}
          symbol& meaning && symbol& meaning \\
          \hline
          $m\in[0,1]$ & ideal decision && $f:\mathcal X\to [0,1]$ & ideal prediction function\\
          $n\in[0,1]$ & recommendation && $g:[0,1]\times  A(\mathcal X)\to [0,1]$& judge's decision rule\\
           $y  \in [0,1]$ & judge's decision  && $A:\mathcal X\to \{0,1\}^D\odot\mathcal X$ & attention function
           \\
          $\mathcal X\subset \R^D$ & feature space && $x_i'=A(x_i)$ & features used by judge \\
          $x_i\in\mathcal X$ & features of individual $i$ 
     \end{tabular}
     \caption{Notation Summary}
     \label{tab:notation}
 \end{table}
 
\subsection{Linear in log odds model}\label{sec:llo}
We motivate a concrete form for a human decision function with the following observations: psychological and behavioral studies have documented that humans do not usually treat probabilities linearly. People tend to overweigh small probabilities and underweigh large probabilities \citep{kahneman_prospect_1979}. Moreover, people can be more sensitive to a change in probability near extreme values---the change from 99\% to 100\% is perceived as more substantial than the change from 33\% to 34\%.  

A probability weighting function, $w(p)$, is a function that maps the $[0,1]$ interval to itself. It represents how people perceive probabilities cognitively in making decisions. A large body of psychology research has been devoted to discern the shape and properties of the probability weighting function. One of the most consistent regularities found in empirical research is the inverse-S-shape of the function \citep{camerer1995individual}. The pioneering work in this field assumed a one-parameter weighting function that exhibits an inverse-S-shape \citep{kahneman_prospect_1979, tversky1992advances}. In a subsequent work, \citet{gonzalez1999shape} proposed the ``linear in log odds'' (LLO) weighting function that has two parameters capturing two logically independent psychological properties as follows:

\begin{align*}
    w(p)= \frac{\delta p^\gamma}{\delta p^\gamma + (1-p)^\gamma}
\end{align*}
where $\gamma$ controls ``discriminability'' (curvature) and $\delta$ controls ``attractiveness'' (elevation).

Discriminability refers to how well the decision maker can distinguish between intermediate probabilities. Consider two extreme cases: a weighting function close to a step function and a weighting function that is almost linear (see left panel of Figure~\ref{fig:lloparameter}). A child's understanding of probability could be closer to the step function---she is able to detect ``certainly will not'' and ``certainly will'' but all other levels fall into the ``maybe'' category. A decision maker with a linear weighting function is more sensitive to small changes in probabilities in most ranges, except near $0$ and $1$. For example, experts have relatively linear weighting functions when betting in their domain of knowledge.
Attractiveness represents the overall elevation of the function. If for all $p$, individual 1 perceives the probability to be larger than individual 2, i.e., $w_1(p) \geq w_2(p)$ for all $p$ with at least one strict inequality, then we say individual 1 finds the gamble more ``attactive'' than individual 2 (see right panel of Figure \ref{fig:lloparameter}).

\begin{figure}[h]
    \centering
\subfigure
{  
\pgfplotsset{xmin=0, xmax=1, ymin=0, ymax=1}    
\begin{tikzpicture}[scale=0.8]
\begin{axis}[domain=0:1,
    ylabel near ticks,
    xlabel near ticks,
    xlabel={$p$},
    ylabel={$w(p)$}]
\addplot[black,dashed,domain=0:1,smooth]  plot ({\x}, {\x });
\addplot[blue,thick,domain=0:1,smooth,samples=500]  plot ({\x}, { \x^(1/10)/(\x^(1/10) + (1-\x)^(1/10)) });
\addplot[red,thick,domain=0:1,smooth,samples=500]  plot ({\x}, { \x^(9/10)/(\x^(9/10) + (1-\x)^(9/10)) });
\node[anchor=south] at (4cm,3.5cm) {$\color{red}w_1$};
\node[anchor=north] at (5.5cm,3cm) {$\color{blue}w_2$};
\end{axis}
\end{tikzpicture}

}  
\subfigure
{  

\pgfplotsset{xmin=0, xmax=1, ymin=0, ymax=1}    
\begin{tikzpicture}[scale=0.8]
\begin{axis}[domain=0:1,
    ylabel near ticks,
    xlabel near ticks,
    xlabel={$p$},
    ylabel={$w(p)$}]
\addplot[black,dashed,domain=0:1,smooth]  plot ({\x}, {\x });
\addplot[blue,thick,domain=0:1,smooth,samples=500]  plot ({\x}, { .5*\x^(5/10)/(.5*\x^(5/10) + (1-\x)^(5/10)) });
\addplot[red,thick,domain=0:1,smooth,samples=500]  plot ({\x}, { 1.1*\x^(5/10)/(1.1*\x^(5/10) + (1-\x)^(5/10)) });
\node[anchor=south] at (4cm,3.5cm) {$\color{red}w_1$};
\node[anchor=north] at (6cm,3cm) {$\color{blue}w_2$};
\end{axis}
\end{tikzpicture}

}
    \caption{Left panel: two probability weighting functions that differ primarily in curvature ($\gamma$). Right panel: two probability weighting functions that differ primarily in elevation ($\delta$).}
    \label{fig:lloparameter}
\end{figure}
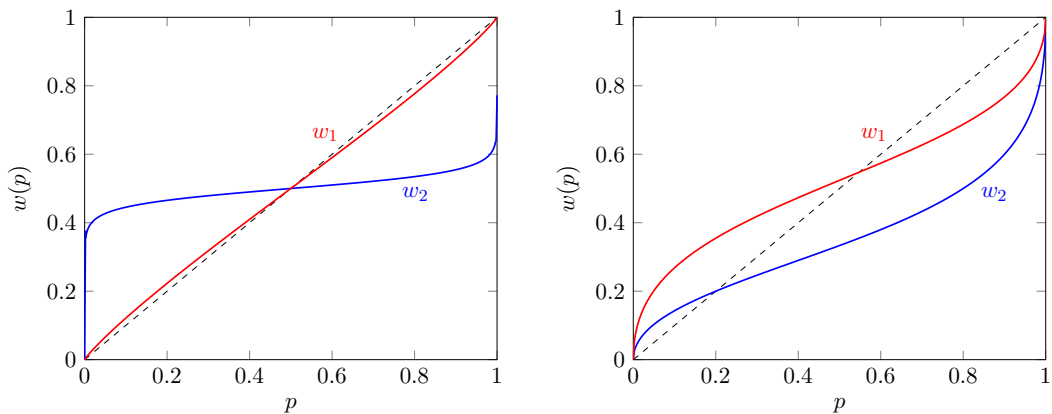

An algorithm's recommendation can be viewed as a probability in many decision-aid applications. In a pre-trial detention setting, the algorithm estimates the probability of a defendant appearing in the trial; in a lending setting, the probability of repayment; in a hiring setting, the probability an applicant is qualified. Therefore, when human decision makers process these recommendations, they could suffer from the same cognitive bias. 
An algorithm that learns about the decision maker can modify its recommendations to help correct this cognitive bias and improve the quality of decisions. 

As an ongoing example, we will consider a linear in log odds decision function. 
Given observed features $x'$, 
\begin{align}\label{eq:llo}
    g(n,x')= \frac{\delta(\beta^\top x') n^\gamma}{\delta(\beta^\top x') n^\gamma + (1-n)^\gamma},
\end{align}
where $\gamma \in [0,1]$ is a parameter that controls the overall curvature of the judge's decision function, $\beta\in\R^{D'}$ is a parameter  linearly determining the effect of observed features $x'$, and $\delta: \R \to [0,\infty)$ is a strictly increasing function. Together, $\beta$ and $\delta$ capture the differential treatment (leniency) by the judge based on observable features. For two individuals with features $x_1$ and $x_2$ suppose that $\delta(\beta^\top x_1') > \delta(\beta^\top x_2')$. Then, with the same recommendation $n$, the judge will always choose a higher (more preferential) decision for individual 1. 

\subsection{Examples}
\subsubsection{Distortion of probabilities}
Consider a loan manager deciding whether to approve loan applications. As most humans, this loan manager systematically underweighs large probabilities of paying back the loan and makes lending decisions that are too conservative for the company to maximize its profit. Suppose the manager's decision function is LLO which satisfies $\delta(\beta^\top x')=1$ and $\gamma=0.5$ for an applicant with observed features $x'$. The function $g(n,x')=\frac{n^{0.5}}{n^{0.5} + (1-n)^{0.5}}$ is plotted in the left panel of Figure~\ref{fig:lloexamples}. Using a richer set of features of the applicant, the algorithm estimates that there is a $0.8$ chance that the applicant will pay back the loan. If the algorithm gives the manager the true estimates, however, the manager will underweigh this probability according to his decision curve and output a decision corresponding to a $0.67$ chance of paying back the loan. Instead, if the algorithm has learned the decision curve of the manager, it can recommend a probability that is higher than the true estimates, $0.94$ in this case, to bring the manager's final perception back to $0.8$. 

\subsubsection{Biases based on group membership}
Consider a judge making pre-trial release decisions. Suppose this judge only observes (or cares about) the gender of the defendants. In other words, $x'$ is a scalar which equals to $1$ if the defendant is female and $0$ if male. The judge's decision functions in these two cases are depicted in the right panel of Figure \ref{fig:lloexamples}. Everything else equal, the judge is more lenient to female defendants since the same recommendation would translate to more positive decisions on the decision curve for females. We do not hypothesize on the origin of such biases---the judge could be discriminatory against certain groups, but it could also be that the judge is operating on his best knowledge of the population average of the two groups. Indeed, women were more likely to appear in court as scheduled if released \citep{steury1990gender}. 

In the presence of group-based bias, the algorithm needs to learn both curves and make recommendations to the judge differently based on group memberships. Suppose that there are two defendants who should receive the same decision, say, $0.4$, but are of different gender. In order to achieve a final decision of $0.4$ for both defendants, the algorithm needs to tune down the recommendation for the female defendant ($0.27$) but exaggerate the recommendation for the male defendant ($0.64$).

\begin{figure}[h]
    \centering
\subfigure%
{%
\pgfplotsset{xmin=0, xmax=1, ymin=0, ymax=1} 
\pgfplotsset{compat=1.17}
\begin{tikzpicture}[scale=0.8]
\begin{axis}[domain=0:1,
    ylabel near ticks,
    xlabel near ticks,
    xlabel={$n$},
    ylabel={$y$},
    clip mode=individual]
\addplot[black,dashed,domain=0:1,smooth]  plot ({\x}, {\x });
\addplot[blue,thick,domain=0:1,smooth,samples=500]  plot ({\x}, { \x^(5/10)/(\x^(5/10) + (1-\x)^(5/10)) });
\node [label={180:{0.67}},circle,fill,scale=0.5] at (0,0.67) {};
\node [label={270:{0.94}},circle,fill,scale=0.5] at (0.94,0) {};
\draw[red, loosely dotted, thick] (0.8,0)--(0.8,0.8);
\draw[red, loosely dotted, thick] (0,0.67)--(0.8,0.67);
\draw[red, loosely dotted, thick] (0,0.8)--(0.94,0.8);
\draw[red, loosely dotted, thick] (0.94,0)--(0.94,0.8);
\end{axis}
\end{tikzpicture}
}  
\subfigure%
{%
\pgfplotsset{xmin=0, xmax=1, ymin=0, ymax=1}    
\pgfplotsset{compat=1.17}
\begin{tikzpicture}[scale=0.8]
\begin{axis}[domain=0:1,
    ylabel near ticks,
    xlabel near ticks,
    xlabel={$n$},
    ylabel={$y$},
    clip mode=individual]
\addplot[black,dashed,domain=0:1,smooth]  plot ({\x}, {\x });
\addplot[blue,thick,domain=0:1,smooth,samples=500]  plot ({\x}, { .5*\x^(5/10)/(.5*\x^(5/10) + (1-\x)^(5/10)) });
\addplot[red,thick,domain=0:1,smooth,samples=500]  plot ({\x}, { 1.1*\x^(5/10)/(1.1*\x^(5/10) + (1-\x)^(5/10)) });
\node[anchor=south] at (4cm,3.5cm) {$\color{red}g(n,1)$};
\node[anchor=north] at (6cm,3cm) {$\color{blue}g(n,0)$};
\node [label={60:{0.27}},circle,fill,scale=0.5] at (0.27,0) {};
\node [label={60:{0.64}},circle,fill,scale=0.5] at (0.64,0) {};
\draw[red, loosely dotted, thick] (0.27,0)--(0.27,0.4);
\draw[red, loosely dotted, thick] (0,0.4)--(0.64,0.4);
\draw[red, loosely dotted, thick] (0.64,0)--(0.64,0.4);
\end{axis}
\end{tikzpicture}
}
    \caption{Left panel: a loan manager exhibiting distortion of probabilities. To reach a final decision of $0.8$, the algorithm should recommend $0.94$. Right panel: a judge exhibiting biases based on group membership, where the red curve $g(n,1)$ represents the judge's decision function for females and the blue curve $g(n,0)$ for males. In order to reach the same final decision of $0.4$, the algorithm needs to tune down its recommendation for females and exaggerate its recommendation for males.}
    \label{fig:lloexamples}
\end{figure}
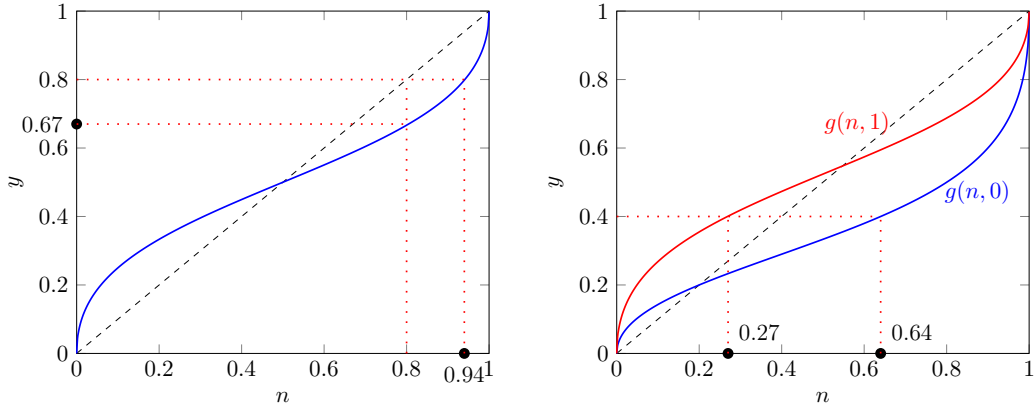

\subsection{Design choices}
\paragraph{Continuous recommendations and decisions.} 
Our model considers the setting where the ideal decision, the algorithm's recommendation, and the judge's decision all take continuous values. In many real-life applications, however, decision makers deliver binary decisions---``to release'' or ``to detain'' for the pretrial judge, ``to approve'' or ``to reject'' for the lending officer. Nonetheless, a continuous scoring coupled with a threshold rule are often underlying such decisions, and our setting still applies in those cases. Compared to the binary setting, controlling continuous decisions is a more precise task, but we also observe richer information and are able to make continuous recommendations. We mention in Section~\ref{sec:discussion} that it will be an interesting extension to study a setting where the goal is to control a binary decision, but the algorithm observes less information and faces more stringent constraints in its recommendations.

\paragraph{Choice of the LLO model.}
In the literature that follows  prospect theory, there have been various ways of modelling of the probability weighting function. We choose to base our model on \citet{gonzalez1999shape} for two reasons. First, the two-parameter model is more expressive than its one-parameter predecessors and can capture more behavioral patterns if we vary the parameters independently. Second, this model provides a psychological interpretation to the parameters, providing better intuition when we vary one of the parameters. That being said, we use the LLO model purely for expositional purposes. Our main results in controllability, identification, and simultaneous game assume no specific functional forms and thus can be applied to any functions satisfying the assumptions. 

\paragraph{Algorithm knows ideal decision without noise.}
This paper proves results about controllability and identifiability of the human decision functions in a noiseless setting where the algorithm also knows the ideal decision perfectly. We suggest considering our results as the ``fundamental limit'' of controllability and identifiability. By distinguishing between what is possible and impossible in the best case, such limits lay the foundation for understanding performance once noise or uncertainty is introduced.

\paragraph{Judge without private information.}
To some readers, the assumption that the human decision maker observes no private information may not seem very palatable. We suggest thinking about this as a setting where the human decision makers observe no more than what the algorithm observes, or where all private information of the decision maker can be fed into the algorithm. We recognize that a promising extension to this paper is to develop methods for identifying and correcting judges' biases when they do have private information. We discuss this in more detail in the future work section.

\section{Controllability}
In examples presented in Figure~\ref{fig:lloexamples}, the algorithm is able to influence the judge to make the desirable decisions (0.8 and 0.4 in the two cases) by modifying its recommendations. Indeed, the algorithm would have been able to do so for any desirable final decision in $[0,1]$. In this case, we say that the algorithm attains \emph{full control} of the judge. We will define and categorize controllability formally in this section.  

Consider general prediction and decision functions $f$ and $g$, whose values are known for any $x \in \mathcal{X}$ and $n \in [0,1]$. Like before, we also assume that the attention function $A$ is known. 

\begin{definition}[Full control]
The algorithm attains full control of the judge at $x$ if there exists an $n \in [0,1]$ such that $g(n,A(x))=f(x)$. If the algorithm attains full control of the judge at any $x \in \mathcal{X}$, then we say that the algorithm attains full control (without specifying $x$).
\end{definition}

\begin{proposition}[Conditions for full control]
\hspace{1em}
\begin{enumerate}
    \item (Full range) A sufficient condition for the algorithm to attain full control of the judge at $x$ is that the range of $g(\cdot,A(x))$ is equal to the entire $[0,1]$ interval. Mathematically,
    \begin{align*}
        [0,1]= \{y\mid y=g(n,A(x)),  n\in [0,1]\}
    \end{align*}
    \item (Limited range) A necessary and sufficient condition for the algorithm to attain full control of the judge at $x$ is that the ideal decision for $x$ is in the range of $g(\cdot,A(x))$. Mathematically,
    \begin{align*}
        f(x) \in \{y\mid y=g(n,A(x)),  n\in [0,1]\}
    \end{align*}
\end{enumerate}
\end{proposition}

\begin{proof}
    We prove statement (2) first. For the sufficient direction, if $f(x)$ is in the range of $g(\cdot,A(x))$, then by the definition of range, there exists some $n \in [0,1]$ such that $g(n,A(x))=f(x)$. This is exactly the definition of attaining full control at $x$. For the necessary direction, suppose on the contrary $f(x)$ is not in the range of $g(\cdot,A(x))$. Then by the definition of range, there does not exist any $n \in [0,1]$ such that $g(n,A(x))=f(x)$, which is the negation of the definition of full control at $x$. Statement (1) follows from the sufficient direction of statement (2) since $f(x)\in [0,1]$ by our assumption. 
\end{proof}

Figure \ref{fig:control} gives a graphical intuition for different cases of controllability. Suppose that only information about an individual $i$ that the judge uses is membership in one of three different population subgroups $\{B,R,G\}$, i.e., $A(x_i) \in \{B,R,G\}$. His decision functions with respect to the subgroups are plotted in the respective color. Note that in our assumption, the ideal decision is determined by the richer feature vector $x_i$. Since different $x_i$ can give rise to the same observable feature $x_i'=A(x_i)$, the ideal decision can be different for defendants belonging to the same subgroup. In the plot, the range of ideal decisions for defendants in subgroup $R$ (resp. $G$) represented by the red (resp. green) shaded area on the y-axis. 

The blue curve represents the case of full control with full range. Since the range of the function $g(n,B)$ is the entire $[0,1]$ interval, any ideal decision falls into the range of $g(n,B)$ and therefore can be achieved by some recommendation. 

The green curve represents the case of full control with limited range. Although the range of judge's decision function $g(n,G)$ is limited, it still covers the possible range of ideal decisions for all the $x_i$ that give the observable feature $G$. The judge is still controllable because any ideal decision regarding this group will not be in the range that the judge refuses to go to. For example, suppose that the ideal decision for any murder defendant, irrespective of other features, is a harsh sentence. Then a judge that only observes the crime type and considers only harsh sentences for murder defendants will agree with the algorithm. 

On the other hand, the red curve illustrates the case when the algorithm fails to attain full control. The judge operates according to decision function $g(n,R)$, but the range of ideal decisions falls outside of the range of function $g(n,R)$. For example, if the ideal decision is $0.8$ for an individual that the judge observes as $R$, then there is no recommendation that the algorithm could give to result in this decision. 

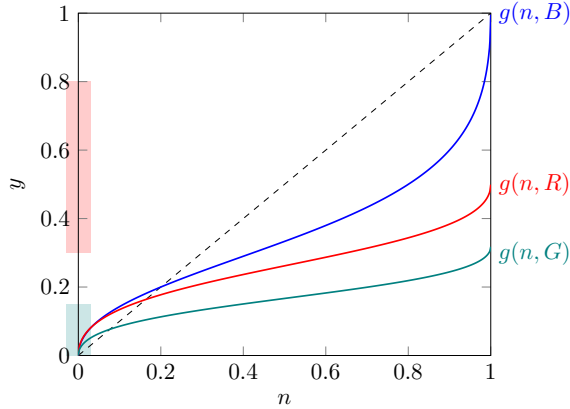
\begin{figure}
    \centering
\pgfplotsset{xmin=0, xmax=1, ymin=0, ymax=1}    
\pgfplotsset{compat=1.17}
\begin{tikzpicture}[scale=0.8]
\begin{axis}[domain=0:1,
    ylabel near ticks,
    xlabel near ticks,
    xlabel={$n$},
    ylabel={$y$},
    clip mode=individual]
\addplot[black,dashed,domain=0:1,smooth]  plot ({\x}, {\x });
\addplot[blue,thick,domain=0:1,smooth,samples=500]  plot ({\x}, { .5*\x^(5/10)/(.5*\x^(5/10) + (1-\x)^(5/10)) });
\addplot[red,thick,domain=0:1,smooth,samples=500]  plot ({\x}, { 0.55*\x^(5/10)/(1.1*\x^(5/10) + (1-\x)^(5/10)) });
\addplot[teal,thick,domain=0:1,smooth,samples=500]  plot ({\x}, { 0.35*\x^(5/10)/(1.1*\x^(5/10) + (1-\x)^(5/10)) });
\begin{pgfonlayer}{background}
  \fill[color=red!20] (axis cs:-0.03,0.3) rectangle (axis cs:0.03,0.8);
  \fill[color=teal!20] (axis cs:-0.03,0) rectangle (axis cs:0.03,0.15);
\end{pgfonlayer}
\node[anchor=west] at (1,1) {$\color{blue}g(n,B)$};
\node[anchor=west] at (1,0.5) {$\color{red}g(n,R)$};
\node[anchor=west] at (1,0.3) {$\color{teal}g(n,G)$};
\end{axis}
\end{tikzpicture}
    \caption{Examples where the algorithm can achieve full control with full range, full control with limited range, or fails to control. The blue, red, and green curves represent the judge's decision functions when observing the features $B$, $R$, and $G$ respectively. The shaded red area on the y-axis is the set of ideal decisions for defendants in group $R$, $\{f(x_i) \mid A(x_i) = R\}$. The shaded green area is the same set for defendants in group $G$, $\{f(x_i) \mid A(x_i) = G\}$. }
    \label{fig:control}
\end{figure}

\subsection{The LLO model: full control}
In this section, we show that the LLO model as shown in equation \eqref{eq:llo} satisfies the full control condition. Moreover, there is a unique recommendation $n$ for every $x$ that produces the ideal decision. 

\begin{definition}\label{def:gxi}
    We define a univariate function $g_{x'}(n): [0,1] \to [0,1]$ such that $g_{x'}(n) = g(n,x')$ for given $x'$. 
\end{definition}

For the LLO model, recall that
    \begin{align*}
         g_{x'}(n) = \frac{\delta(\beta^\top x') n^\gamma}{\delta(\beta^\top x') n^\gamma + (1-n)^\gamma}\:.
    \end{align*}

\begin{proposition}[Full control in the LLO model] \label{fullcontrolllo}
    The algorithm attains full control of the decision maker in the LLO model. Moreover, the recommendation that attains control is unique and equals $g_{x'}^{-1}(f(x))$.
\end{proposition}

\begin{proof}

For the LLO model, the condition for full control is:
    \begin{align*}
        \forall x \in \mathcal{X}, \; \exists n \in (0,1) \text{  s.t. }  \frac{\delta(\beta^\top x') n^\gamma}{\delta(\beta^\top x') n^\gamma + (1-n)^\gamma}=f(x') .
    \end{align*}
    
For any given $x$, the right hand side is a fixed number in $[0,1]$. On the left hand side, we have $g_{x'}(0)=0$ and $g_{x'}(1)=1$, and $g_{x'}(n)$ is continuous in $n$ in $[0,1]$. By the intermediate value theorem, there exists $n\in [0,1]$ such that $g_{x'}(n)=f(x)$. Existence is established. If we further use the fact that $g_{x'}(n)$ is strictly increasing in $[0,1]$, then $g_{x'}(n)$ is invertible, this $n$ is unique and equals $g_{x'}^{-1}(f(x))$.
\end{proof}

\section{Parameter Identification}
In the previous section, the algorithm's ability to control the judge rests on one crucial assumption---the decision rule of the judge is known to the algorithm. In this part, we explore when the algorithm can learn about judge's decision function from past interactions with the judge.

We study the case in which the decision function of the judge takes a particular parametric form.
In particular, we suppose that it is determined by the parameters $\beta\in\mathbb R^{D'}$ and $\gamma\in\mathbb R$:
\newcommand{\g}{h}
\[g(n, x') = \g(\beta^\top x', \gamma, n)\]
where $\g:\mathbb R^3 \to [0,1]$ is a nonlinear function whose form is known.
The first argument of $\g$ represents the contribution of the observed features $x'$.
This contribution depends on an unknown parameter $\beta$.
The second argument of $\g$ is an unknown scalar parameter $\gamma$.
The final argument of $\g$ is the algorithm's recommendation, which is known.
In the following section, we use the LLO model as an example of such a function $\g$.

The parameters are unknown, but the judge's decisions, features, and algorithm recommendations can be observed. 
The dataset has the form $\{x_i, x_i', n_i, y_i\}_{i=1}^N$ where $y_i = \g(\beta^\top x_i', \gamma, n_i)$.
We assume that  $\beta\in\mathcal B$ and $\gamma\in\mathcal G$ which are simply connected and compact subsets of $\mathbb{R}^{D'}$ and $\mathbb{R}$ respectively.
We make the following assumptions about $\g$.
\begin{assumption}
    The function $\g$ is proper and differentiable with respect to its first and second argument.
    We denote the partial derivative with respect to the first argument
    and with respect to the second argument as
    $$\frac{\partial }{\partial u} \g(u, \gamma, n) =: \g'(u, \gamma, n)\quad\text{and}
    \quad\frac{\partial }{\partial \gamma} \g(u, \gamma, n) =: \g_\gamma(u, \gamma, n).$$
    We further assume that $\g$ is strictly increasing with respect to its first argument, i.e. $\g'(u, \gamma, n)>0$ for all values of $u,\gamma,n$.
\end{assumption}
The assumption that $h$ is strictly increasing in its first argument corresponds to the inner product $\beta^\top x'$ being positively correlated with more positive decisions.
This aligns with the discussion of the LLO model in Section~\ref{sec:llo}.


We present an identification result the develops conditions on the observed data and the function $h$ under which the unknown parameters can be uniquely identified.
This result is a special case of a more general result that we prove in Appendix~\ref{app:identification}.

\newcommand{\rank}{\mathsf{rank}}
\newcommand{\colspace}{\mathcal{C}}
\begin{proposition}\label{prop:identify}
    Define the feature matrix and vector of derivative ratios as
    \begin{align}
        \label{eq:identification-mx}
        X = \begin{bmatrix}
         x_1'^\top \\ \vdots \\   x_N'^\top \end{bmatrix}
        \quad\text{and}\quad
        d_{\beta,\gamma} = \begin{bmatrix}
            \g_\gamma(\beta^\top x_1', \gamma, n_1)/\g'(\beta^\top x_1', \gamma, n_1)\\
            \vdots \\ 
            \g_\gamma(\beta^\top x_N', \gamma, n_N)/\g'(\beta^\top x_N', \gamma, n_N)
        \end{bmatrix}
    \end{align}
    Then the parameters $\beta$ and $\gamma$ can be uniquely identified from the sample if the following two conditions hold:
    \begin{itemize}
        \item Rank condition: $X$ is full rank, i.e. $\rank(X) = D$, and $d_{\beta,\gamma}$ is nonzero for all $\beta\in\mathcal B$ and $\gamma\in\mathcal G$.
        \item Independence condition: $ d_{\beta,\gamma}$ is not contained within column space of $X$, i.e. $ d_{\beta,\gamma}\notin \colspace(X)$ for all $\beta\in\mathcal B$ and $\gamma\in\mathcal G$.
    \end{itemize}
\end{proposition}

The feature matrix $X$ is familiar from linear parameter identification,
as is the rank condition on it.
The vector of derivative ratios $d_{\beta, \gamma}$ captures a notion of variation in the nonlinear decision function over the data samples.
The condition $d_{\beta, \gamma}\neq 0$ is equivalent to requiring that the derivative with respect to $\gamma$ is nonzero for at least one datapoint.
The independence condition captures the idea that we must observe variation in the sensitivity of the decision function to  $\beta$ and $\gamma$ that is not due to variation in observed features $x_i'$. Such variation can arise due to the variation in recommendations $n_i$. In the following subsection we present a sufficient condition for the LLO model that further elucidates this requirement.




\subsection{The LLO model: identification}

The general results above apply to the LLO model and can be specialized to build further intuition about the required conditions.
The following corollary is a sufficient (but not necessary) condition for identifiability.

\begin{corollary}
Suppose that the judge's decisions are determined by the LLO model~\ref{eq:llo} and that
$\delta(\cdot)$ is strictly increasing.
Further assume that the dataset contains two individuals $i$ and $j$ with identical observed features $x_i'=x_j'$ but different recommendations $n_i\neq n_j$.
Then the parameters $\beta$ and $\gamma$ can be uniquely identified as long as $X$ is full rank.
\end{corollary}

We remark on this condition.
One way to achieve it is for the algorithm to sometimes inject noise into its recommendations.
However, it may not be realistic or desirable to allow for such ``exploration,'' potentially at the expense of accuracy for particular decisions.
Thus, we can consider the case that
$n_i = f(x_i)$, i.e., the algorithm is truthful in the period when we collect the training data. 
If the judge uses a strict subset of the full features so that $x_i' \neq x_i$, it may very well be that the recommendations differ $n_i\neq n_j$ (since the full features differ $x_i\neq x_j$) while the observed features $x_i'=x_j'$ are the same. 
For example, if a lending officer only observes an applicant's gender and the algorithm recommends their credit scores, then it 
is only necessary that the algorithm has seen the outcomes for applicants of the same gender with different recommended credit scores.

\begin{proof}
We begin by computing partial derivatives.
Note that 
\[\g'(u, \gamma,n) = 
\frac{\delta'(u)n^\gamma  (1-n)^\gamma
}{(\delta(u)n^\gamma + (1-n)^\gamma)^2}\]
which is always strictly positive since $\delta$ is strictly increasing and thus $\delta'>0$. Also,
\[ \g_\gamma(n,u, \gamma) = 
\frac{ \delta(u)n^\gamma (1-n)^\gamma  (\log(1-n)-\log n) }{(\delta(u)n^\gamma + (1-n)^\gamma)^2}
\]
The ratio of derivatives is: 
\[(\log(1-n)-\log n)\frac{  \delta(u)}{\delta'(u)
}\]
which is independent of $\gamma$.
By assumption, the dataset contains individuals with identical observed features $x_i'=x_j'$.
Let $u_i=\beta^\top x_i'=\beta^\top x_j'$.
The corresponding indices of $d_{\beta,\gamma}$ are not equal since
\[(\log(1-n_i)-\log n_i)\frac{  \delta(u_i)}{\delta'(u_i)
}\neq (\log(1-n_j)-\log n_j) \frac{ \delta(u_i)}{\delta'(u_i)
}\]
where the inequality holds because the function $\log(1-n)-\log(n)$ is strictly monotone in $n$ and the recommendations are different, i.e., $n_i\neq n_j$.

Finally notice that because the observed features are the same, the rows $i$ and $j$ of $X$ are identical. Thus, for any element of the column space of $X$, it must be that the value at indices $i$ and $j$ are identical. Therefore,
$d_{\gamma,\beta}\notin \mathcal C(X)$ for all $\beta$ and $\gamma$ and Proposition~\ref{prop:identify} applies.
\end{proof}

\section{Game with a Strategic Judge}\label{sec:strategic}
In the previous sections, we have considered the feasibility of learning and controlling a naïve judge---a judge who follows a fixed decision rule and is not aware of the algorithm's manipulations. This section introduces a strategic judge who is able to change his decision rule to best respond to the algorithm.

\subsection{Game set up}
We now consider the algorithm and judge as two players. 
As before, for each individual decision subject $i$, the algorithm observes a high-dimensional feature vector $x_i$
while the judge only observes less informative features $x_i'$.

While we previously considered $m_i$ to denote the ideal decision for individual $i$, 
we now explicitly distinguish between
information and action in order to consider a judge and an algorithm with different objectives.
We now refer to $m_i$ to be the ``true state'' of individual $i$ and as before
assume that the algorithm knows the function $f:\mathcal{X} \to [0,1]$  which gives
$m_i=f(x_i)$. 
The judge does not observe $m$ but has a prior on the distribution of the states with density $\mu(m)$.

Based on the true state $m$, 
the ideal decision (in the eyes of the algorithm) is determined by
a twice continuously differentiable utility function $U^A(y,m)$, where as before $y \in [0,1]$ is the final decision made by the judge upon receiving the algorithm's recommendation.
The preferred decision of the judge depends on the true state in addition to observed features $x'$.
We define a scalar ``bias'' parameter $b_i = b(x_i')$ which measures how closely the two players' interests coincide. 
Then, the judge has a twice continuously differentiable utility function $U^J(y,m,b)$. 


\begin{assumption}[Utility functions]\label{assum:utilityfunc}
    Denoting partial derivatives by subscripts in the usual way, we assume that, for each $m$,
    \begin{enumerate}
        \item For $i=A, J$, $U^i_{11}(\cdot)<0$. $U^i$ has a unique maximum in $y$ for each given $(m,b)$ pair.
        \item For $i=A, J$, $U^i_{12}(\cdot)>0$. Both players prefer a higher decision when the true state is higher. 
        \item $U^J(y,m,0) = U^A(y,m)$, $b \geq 0$, and $U^J_{13}(\cdot)>0$ everywhere. The utility-maximizing decision of the judge is weakly higher than that of the algorithm for any $m$, and an increase in $b$ shifts the judge's preference away from the algorithm's everywhere. 
    \end{enumerate}
\end{assumption}

The game proceeds as follows. When an individual $i$ arrives, the algorithm observes features $x_i$ and thus the true state of the world $m_i$. The judge observes partial features $x_i'$ and thus his bias $b_i$. The algorithm sends a recommendation $n_i \in [0,1]$ to the judge according to some recommendation rule $q_{x_i'}(n\mid m)$. The judge processes the information in this recommendation and makes a decision $y$ according to some decision rule $g_{x_i'}(n)$. The judge's decision affects both players' payoffs. 

\begin{assumption}[Information Structure]
\hspace{1em}
\begin{enumerate}
    \item We assume that the judge does not know the true state $m$ and infers it in a Bayesian way from the recommendation $n$ he receives.\footnote{We assume that the judge does not try to update his prior about $m$ from $x'$. 
    There are two motivations for this assumption: 1) the judge does not know $f$ and therefore observed features are not useful to him for determining the state of the world.
    2) the full features $x$ are much richer compared to $x'$ so that $x'$ alone does not give information about $m$. In other words, $m$ is independent of $x'$.}
    \item The algorithm observes both $x$ and $x'$ so the bias $b$ is common knowledge.
\end{enumerate}
\end{assumption}

\paragraph{Remark.} Since the bias parameter $b_i$ parameterizes the judge's utility function and is determined by the observed features $x_i'$ (e.g. a demographic category), each $x_i'$ induces a game specific to the group of individuals with the same $x_i'$. To characterize equilibria, we only need to analyze one of such games. We drop the subscript $i$ and $x_i'$ on the strategies of the players whenever it is clear that we are analyzing one specific game. For the same reason, restricting $b\geq 0$ in Assumption~\ref{assum:utilityfunc} is without loss of generality, because if it were instead $b(x_i')<0$ (an overall negative bias of the judge), the analysis below could be repeated for that group with reversed signs. 

\subsection{The LLO model as a sequential game with alternate movements}
To build intuition for the behavior of a strategic judge, in this section we
adapt the LLO model and the controllability result from the previous sections. Before defining the utility functions for the judge and the algorithm, we first define a judge specific ideal decision function. 

\begin{definition}
   The ideal decision for the judge with bias $b$ is a function $g_b(m):[0,1]\to[0,1]$ mapping from a true state of the world to a decision. It represents the utility-maximizing decision of the judge if he knows the true state of the world. 
\end{definition}

Comparing the above to Definition~\ref{def:gxi}, we now perceive the judge as having an ideal decision depending on the true state of the world, rather than a static decision function depending on a recommendation.
Additionally, in this
notation we let the function directly depend on the ``bias'' parameter $b$ (rather than indirectly on the observed features). 

In this section, we assume that the ideal decision function of the judge follows the LLO form with two additional constraints: $\gamma=1$ and $\delta(\beta^\top x_i') \in [1,\infty)$. We define $b=log(\delta)>0$. Then the form of the judge's ideal decision function is:
\begin{align}\label{eq:llogame}
    g_b(m)= \frac{e^b m}{e^b m + (1-m)} =\frac{\delta(\beta^T x_i') m}{\delta(\beta^T x_i') m + (1-m)} 
\end{align}
The curvature is fixed at $1$ while the elevation of the curve is controlled by the bias of the judge, which depends on the defendant's features he observes. The resulting shape of the function is plotted in the left panel of Figure \ref{fig:game}. When $b=0$ (i.e., $\delta=1$), the function is the identity line. When $b$ increases, the curve is elevated in all sections except for the end points $0$ and $1$. 

We discuss the interpretation of this model. Fixing $\gamma$ at $1$ removes the inverse-S-shape of the function, reflecting the different nature of biases in the strategic setting. The judge no longer misinterprets probability; instead, he consciously desires a different outcome for individuals based on their features $x'$. Therefore, it is reasonable for this desire to be in the same direction for all subjects with the same features, regardless of their true states. For example, if a loan officer believes that ``women should get more loans,'' then his ideal decisions will be more positive for women compared to men for any value of the true state. 

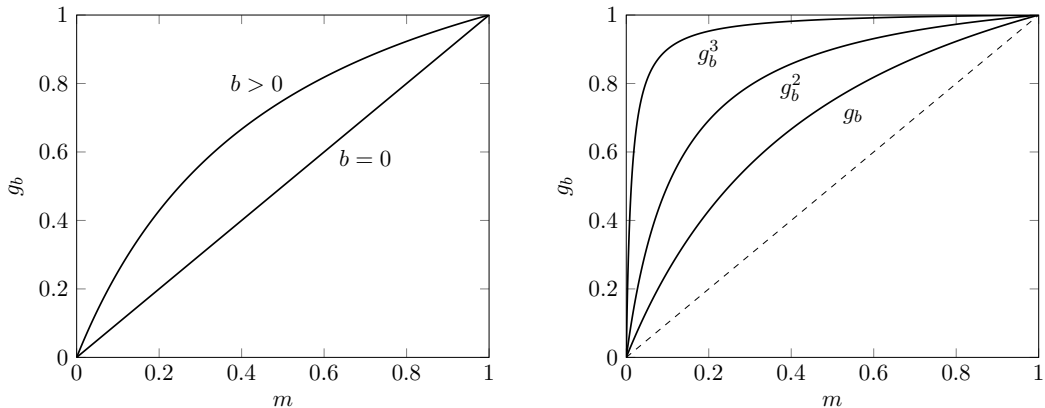
\begin{figure}[h]
    \centering
\subfigure
{  
\pgfplotsset{xmin=0, xmax=1, ymin=0, ymax=1} 
\pgfplotsset{compat=1.17}
\begin{tikzpicture}[scale=0.8]
\begin{axis}[domain=0:1,
    ylabel near ticks,
    xlabel near ticks,
    xlabel={$m$},
    ylabel={$g_b$},
    clip mode=individual]
\addplot[black,thick,domain=0:1,smooth]  plot ({\x}, {\x });
\addplot[black,thick,domain=0:1,smooth,samples=500]  plot ({\x}, { 3*\x/(3*\x + (1-\x)) });
\node[anchor=north] at (0.7,0.63) {$\color{black}b=0$};
\node[anchor=east] at (0.52,0.8) {$\color{black}b>0$};
\end{axis}
\end{tikzpicture}
}  
\subfigure
{  

\pgfplotsset{xmin=0, xmax=1, ymin=0, ymax=1}    
\pgfplotsset{compat=1.17}
\begin{tikzpicture}[scale=0.8]
\begin{axis}[domain=0:1,
    ylabel near ticks,
    xlabel near ticks,
    xlabel={$m$},
    ylabel={$g_b$},
    clip mode=individual]
\addplot[black,dashed,domain=0:1,smooth]  plot ({\x}, {\x });
\addplot[black,thick,domain=0:1,smooth,samples=500]  plot ({\x}, { 3*\x/(3*\x + (1-\x)) });
\addplot[black,thick,domain=0:1,smooth,samples=500]  plot ({\x}, { 9*\x/(9*\x + (1-\x)) });
\addplot[black,thick,domain=0:1,smooth,samples=500]  plot ({\x}, { 81*\x/(81*\x + (1-\x)) });
\node[anchor=north] at (0.55,0.75) {$\color{black}g_b$};
\node[anchor=north] at (0.4,0.85) {$\color{black}g_b^2$};
\node[anchor=north] at (0.2,0.95) {$\color{black}g_b^3$};
\end{axis}
\end{tikzpicture}
}
    \caption{Left panel: The ideal decision function of the judge when $b=0$ and $b>0$. Right panel: The shape of the composition of functions for a typical decision function $g_b$.}
\label{fig:game}
\end{figure}

We are now ready to define the utility functions of the algorithm and the judge for the LLO model.

\begin{definition}
    The algorithm's utility function is $U^A(y,m) = -(y-m)^2$. The judge's utility function is $U^J(y,m,b) = -(y-g_{b}(m))^2$, where $g_b(m)$ is the judge's ideal decision function. The algorithm is best-off if the judge's decision follows the identity line while the judge is best-off when his decision follows his ideal decision curve in the LLO form. 
\end{definition}

We pose the LLO model we discussed in previous sections as a sequential game with alternate movements. This is the case when the two players move as follows:

\begin{enumerate}
    \item $T=1$, the algorithm is constrained to behave ``truthfully'', i.e., $n=m$, while the judge best responds according to his utility. 
    \item $T=2$, the judge's decision rule is fixed as in $T=1$ and the algorithm changes its recommendation rule in best response.
    \item $T=3$, the algorithm's recommendation rule is fixed as in $T=2$ and the judge changes its decision rule in best response.
    \item $T=4$, the judge's decision rule is again fixed as in $T=3$ and the algorithm changes its recommendation rule in best response, and so on and so forth.
\end{enumerate}

In other words, the algorithm and the judge take turns, reacting to the other party's decision or recommendation rule in the last period. For this example, we  assume that both the judge and the algorithm are myopic, maximizing their utility within each period that they are allowed to update. 

At $T=1$, the judge's utility is maximized by definition as the algorithm is mandated to provide the judge with the true state of the world. The judge's decision rule is simply 
$$y=g_{b}(n).$$
At $T=2$, when the algorithm is allowed to change its recommendation rule, it wants the final decision to be as close to $m$ as possible. Given the decision rule of the judge in the first period, the algorithm best responds by choosing 
\begin{align*}
    n=g_{b}^{-1}(m)
\end{align*}
which arrives at the same solution as given in Proposition \ref{fullcontrolllo}. Under this personalized recommendation, the final decision of the judge at  $T=2$ is
\begin{align*}
    y=g_{b}(n) = g_{b}(g_{b}^{-1}(m)) = m
\end{align*}
implying that the algorithm attains full control.

At $T=3$, the algorithm is fixed and the judge is allowed to change his decision rule. The best that the judge can do is deduce the value of $m$ from algorithm's recommendation and then apply his ideal decision function $g_{b}$ to $m$ to obtain his desired outcome. This can be achieved in this case since the algorithm's recommendation function $n = g_{b}^{-1}(m)$ is invertible---the recommendation and the true state of the world has a one-to-one correspondence---so the judge deduces that $m = g_{b}(n)$. Therefore, the judge's decision function is
\begin{align*}
    y=g_{b}(m) = g_{b}(g_{b}(n)) = g_{b}^2(n)
\end{align*}
The right panel of Figure \ref{fig:game} plots a typical $g_{b}$ and its compositions with itself. Note that $g_{b}^2$ is a more distorted LLO curve, i.e., further away from the identity line. 

The pattern continues in subsequent periods. At $T=4$, the algorithm recommends $n=(g_{b}^2)^{-1}(m)$ to perfectly control the judge's decision. At $T=5$, the judge adopts the decision $y=g_{b}^3(n)$ to achieve his first-best. 

At the limit, the judge's decision rule is almost a step function---choose $0$ when algorithm recommends $0$ and choose $1$ when algorithm recommends everything else. The algorithm in turn recommends $0$ for every state of the world except for the state $1$. Intuitively, if the judge starts with a positive bias, the judge ultimately becomes too lenient to almost everyone so the algorithm has to become too stringent. The result is a loss of information from the algorithm and a loss of controllability of the judge.   

\subsection{Simultaneous game}
In this section, we proceed to analyze the simultaneous game in which the algorithm and the judge best respond to each other in the same period. One insight we can glean from the sequential game is that the algorithm will not choose a one-to-one correspondence between $m=f(x)$ and $n$ in equilibrium. Otherwise, the judge can perfectly back out the true state of the world and achieve his first-best. Indeed, with reference to the classical cheap talk game in \citet{crawford1982strategic}, we show that the equilibria in our game are variations of partition equilibria where the algorithm only recommends to the judge the interval containing the true state. 

Formally, the equilibrium of this game consists of a recommendation rule for the algorithm, $q(n\mid m)$, and a decision rule for the judge, $g(n)$, such that:
\begin{enumerate}
    \item For each $m$, $\int_0^1 q(n\mid m)dn = 1$ 
    and if $n^*$ is in the support of $q(\cdot \mid m)$, then $n^*$ solves 
    $$\max_{n\in [0,1]} U^A(g(n),m).$$
    \item For each $n$, $g(n)$ solves 
    $$\max_{y} \int U^J(g(n),m,b)p(m\mid n)dm \quad 
    \text{where}\quad 
    p(m\mid n) \equiv \frac{q(n\mid m)\mu(m)}{\int q(n\mid t)\mu(t)dt}.$$
\end{enumerate} 

\begin{definition}
    For all $m \in [0,1]$, define the utility-maximizing decisions
    \begin{align*}
        y^A(m)= \argmax U^A(y,m)
        \quad \text{and}\quad
        y^J(m,b)= \argmax U^J(y,m,b)
    \end{align*}
    By Assumption~\ref{assum:utilityfunc},
    $y^A(m)$ and $y^J(m,b)$ are well defined and continuous in $m$.
\end{definition}

To make our general results applicable to the LLO model, we need to make a small adjustment to the theorems in \citet{crawford1982strategic} as follows. 

\begin{definition}
    Let $\overline{N}=\{n: g(n)=\overline{y}\}$ be the set of recommendations that could result in the decision $\overline{y}$. We say that a decision $\overline{y}$ is induced by a state $\overline{m}$ if $\int_{\overline{N}}q(n\mid \overline{m})d_n >0$. That is, there is a positive probability that the utility-maximizing recommendation for the algorithm in state $\overline{m}$ will induce the utility-maximizing decision $\overline{y}$ for the judge.
\end{definition}

\begin{lemma}[Adapted from \citet{crawford1982strategic} Lemma 1]\label{lemma1}
    If $y^A(m) \neq y^J(m,b)$ for all $m$ in the interval $[c,d] \subset [0,1]$, then there exists an $\epsilon>0$ such that if $u$ and $v$ are decisions induced in equilibrium by $m \in [c,d]$, then $|u-v|\geq \epsilon$. Further, the set of decisions induced in equilibrium by $m \in [c,d]$ is finite. 
\end{lemma}

\begin{proof}
    The proof follows the same arguments as in the proof of Lemma 1 in \citet{crawford1982strategic}, except that all statements that hold on $[0,1]$ now hold on the closed interval $[c,d]$.
\end{proof}

The lemma proves that in equilibrium, the judge would only choose from a finite set of decisions, indicating that he is not able to obtain more detailed information from the algorithm. It confirms our intuition that the equilibrium recommendation must be imprecise if the agents' interests do not coincide. We will now show that, in particular, the equilibria in this model take a simple partition form. Let us first introduce some notation for describing partition equilibria. 

\begin{definition}
    Let $(a_0,\dots,a_N)$ denote a partition of the interval $[c,d] \subset [0,1]$ with $N$ steps. The end points of those intervals are $c= a_0<\dots<a_N=d$. Define, for all $\underline{a}, \overline{a} \in [c,d]$ and $\underline{a}\leq \overline{a}$,
\begin{align*}
\overline{y}(\underline{a},\overline{a}) = 
    \begin{cases}
    \argmax \int_{\underline{a}}^{\overline{a}}U^J(y,m,b)f(m)dm & \text{if  } \underline{a} < \overline{a} \\
    y^J(\underline{a},b) & \text{if  } \underline{a} = \overline{a}
    \end{cases}
\end{align*}
\end{definition}
In words, $\overline{y}(\underline{a},\overline{a})$ is the decision that maximizes the judge's expected utility given that the true state of the world $m$ falls between the partition interval $[\underline{a},\overline{a}]$. Now, we are ready to present the proposition that characterizes the equilibria of the simultaneous game.

\begin{proposition}[Adapted from \citet{crawford1982strategic} Theorem 1]\label{prop:partitioneq}
    If $y^A(m) \neq y^J(m,b)$ for all $m$ in the interval $[c,d]\subset [0,1]$, then there exists a positive integer $N(b)$ such that, for every $N$ with $1\leq N \leq N(b)$, there exists at least one equilibrium $\{y(n),q(n\mid m)\}$, where $q(n\mid m)$ is uniform, supported on $[a_i, a_{i+1}]$ if $m\in (a_i, a_{i+1})$, 
    \begin{enumerate}
        \item $U^A(\overline{y}(a_i,a_{i+1}),a_i,b)- U^A(\overline{y}(a_{i-1},a_{i}),a_i,b) = 0$ for $i=1,\dots,N-1$,
        \item $g(n)=\overline{y}(a_i,a_{i+1})$  for all $n\in (a_i,a_{i+1})$,
        \item $a_0 = c$ and $a_N = d$.
    \end{enumerate}
    Moreover, any equilibrium is essentially equivalent to one in this class for some value of $N$ with $1\leq N \leq N(b)$.
\end{proposition} 

\begin{proof}
    Similarly, the proof follows essentially the same arguments as in the proof of Theorem 1 in \citet{crawford1982strategic}, except that all results that hold on $[0,1]$ now hold on the closed interval $[c,d]$.
\end{proof}

Proposition~\ref{prop:partitioneq} establishes the existence of multiple partition equilibria from size one to size $N(b)$. In each partition equilibrium, the algorithm recommends the same value for all states within a  partition interval. Condition (1) is an ``arbitrage'' condition, requiring that the choice of end points of the partition intervals must be such that the algorithm is indifferent between recommending the lower or higher interval if the true state falls at the end point. Condition (2) simply states that given the algorithm's recommendation rule, the judge chooses the decision that maximizes his expected utility over the interval that the recommendation lies in.

To give intuition to equilibria of different sizes, let us consider the two extreme cases. The equilibrium of size one is usually called the ``babbling equilibrium'' in literature, since the algorithm gives the same recommendation regardless of the true state and therefore is completely uninformative. On the other hand, the equilibrium of size $N(b)$ is the most informative equilibrium since the richest information about the state of the world is transmitted. Under our assumptions about the utility functions, it can be shown that $N(b)$ is nondecreasing in $b$, the discrepancy between the players' objectives. In other words, the best-case equilibria are more informative only if the objectives of the algorithm and the judge are more closely aligned.

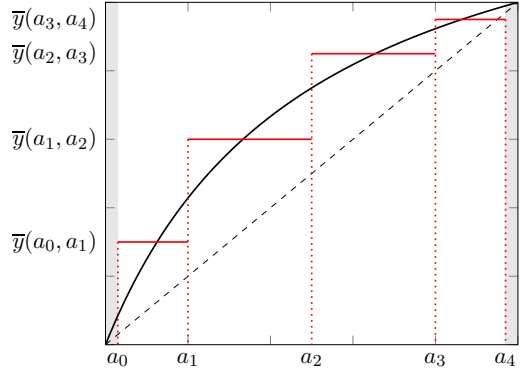
\begin{figure}[h]
    \centering
\pgfplotsset{xmin=0, xmax=1, ymin=0, ymax=1}    
\pgfplotsset{compat=1.17}
\begin{tikzpicture}[scale=0.8]
\begin{axis}[domain=0:1,
    ylabel near ticks,
    xlabel near ticks,
    xticklabels={,,},
    yticklabels={,,},
    clip mode=individual]
\addplot[black,dashed,domain=0:1,smooth]  plot ({\x}, {\x });
\addplot[black,thick,domain=0:1,smooth,samples=500]  plot ({\x}, { 3*\x/(3*\x + (1-\x)) });
\draw[red, thick] (0.03,0.3) -- (0.2,0.3);
\draw[red, thick] (0.2,0.6) -- (0.5,0.6);
\draw[red, thick] (0.5,0.85) -- (0.8,0.85);
\draw[red, thick] (0.8,0.95) -- (0.97,0.95);
\draw[red, dotted, thick] (0.03,0.3) -- (0.03,0);
\draw[red, dotted, thick] (0.2,0.6) -- (0.2,0);
\draw[red, dotted, thick] (0.5,0.85) -- (0.5,0);
\draw[red, dotted, thick] (0.8,0.95) -- (0.8,0);
\draw[red, dotted, thick] (0.97,0.95) -- (0.97,0);
\node[anchor=north] at (0.03,0) {$\color{black}a_0$};
\node[anchor=north] at (0.2,0)  {$\color{black}a_1$};
\node[anchor=north] at (0.5,0)  {$\color{black}a_2$};
\node[anchor=north] at (0.8,0)  {$\color{black}a_3$};
\node[anchor=north] at (0.97,0)  {$\color{black}a_4$};
\node[anchor=east] at (0,0.3) {$\color{black}\overline{y}(a_0,a_1)$};
\node[anchor=east] at (0,0.6) {$\color{black}\overline{y}(a_1,a_2)$};
\node[anchor=east] at (0,0.85) {$\color{black}\overline{y}(a_2,a_3)$};
\node[anchor=east] at (0,0.95) {$\color{black}\overline{y}(a_3,a_4)$};
\begin{pgfonlayer}{background}
  \fill[color=black!10] (axis cs:0,0) rectangle (axis cs:0.03,1);
  \fill[color=black!10] (axis cs:0.97,0) rectangle (axis cs:1,1);
\end{pgfonlayer}
\end{axis}
\end{tikzpicture}
    \caption{Illustrations of a partition equilibrium with four intervals. The area not grayed-out denotes the interval $[\delta,1-\delta]$ that is sufficiently away from end points. The algorithm sends the same recommendation for the states of the world in the same partition interval, and the judge chooses the decision that maximizes the expected utility knowing that the true state falls into this interval.}
    \label{fig:partitioneq}
\end{figure}

\begin{example}[LLO model in the simultaneous game]
    
One feature of the LLO model is that the curve always passes through $(0,0)$ and $(1,1)$ for any parameter value. The two end points complicate our analysis since the objectives of the algorithm and the judge always approach agreement near the end points, regardless of the value of the bias parameter. The interpretation of the end points are those ``unambiguous cases,'' in which both the algorithm and the judge converge to an extreme opinion about the individual. 

We argue that the more interesting cases are the cases sufficiently away from the end points, where the algorithm and the judge have substantial disagreements and thus have incentives to compete with each other. If we restrict our analysis to the interval $[\delta,1-\delta]$ for any $\delta>0$, we can apply Proposition~\ref{prop:partitioneq} and characterize the partition equilibria in the LLO model.  
Figure \ref{fig:partitioneq} gives an example of such an equilibrium. For a given $\delta>0$, the the interval $[\delta,1-\delta]$ is partitioned into four intervals with end points $a_0,\dots,a_4$. The algorithm gives the same recommendation as long as the true state falls within the same partition interval, say $n_1$ for $m\in [a_0,a_1]$, $n_2$ for $m\in [a_1,a_2]$, etc.\footnote{In theory, it does not matter what the content of the recommendation is as long as the two players agree on the recommendation rule. In practice, we do often see recommendations in the form of ``the true state is between $a_0$ and $a_1$.''} When the judge receives the recommendation $n_1$, he chooses $\overline{y}(a_0,a_1)$, the decision that maximizes his expected utility, conditional on knowing that the true state is somewhere in $(a_0,a_1)$.
\end{example}

\section{Discussion and future work}\label{sec:discussion}

In this paper we consider decision-aid algorithms which give recommendations to human decision makers but cannot otherwise impact the final decision.
We develop conditions under which the ``soft'' influence of recommendations is sufficient to ensure that ideal decisions are made, even in the presence of a biased judge.
Our results span two settings: first, we consider a static judge who operates according to a fixed decision rule.
In this setting, we characterize the conditions necessary for full control, when the decision function is known, and identification, when the decision function is unknown but can be learned from data.
We then consider a judge who is aware of potential manipulations by the algorithm and selects a decision rule strategically. 
We draw a connection to ``cheap talk'' games and show how misaligned objectives between the judge and the algorithm coarsen the information transmitted in equilibrium.

Our characterization highlights 
the fundamental limits of algorithmic influence in decision-aid systems.
This influence is limited when the judge and algorithm have very different objectives. 
In the static setting, controllability fails when the range of decisions acceptable to the judge does not contain the algorithm's ideal decision, indicating a fundamental conflict of objectives.  
In the strategic setting,
disagreement about the ideal decision alone is enough to cause a coarsening of the information shared, and consequently a coarsening of the final decision.
In the case of extreme disagreement, the only equilibrium would be a partition of size one (i.e., the ``babbling equilibrium''), where the algorithm sends the same recommendation regardless of the true state and the judge makes a single decision for all. 
Intuitively, if the objective of the algorithm is too different from that of the judge, the judge is more likely to ignore the algorithm despite its informational advantages.

The interpretation of these observations and their broader implications depends on
the context of the decision problem and the reasons for disagreement between the two parties.
Indeed, the spectre of algorithmic-control-via-recommendation is not entirely palatable. 
Consider an algorithm which is claimed to be a decision-aid but is in fact secretly designed to manipulate final decisions.
The human judges, on the surface, continue to have the final say over decisions and are held accountable for their decisions.
Meanwhile, the goals of the algorithm's designer are realized without any of the accountability. 
When viewed in this light, our impossibility results help to understand the extent to which we should worry about such ``accountability laundering.''



An important direction for future work is to relax the assumption that the algorithm has access to a superset of the information available to the human decision maker.
Then the problem becomes how to account for the fact that humans may observe private information but may also have biases. 
There are two potential ways to solve this problem. One possibility is for the algorithm to appropriately query the human and solicit his private information. Then the algorithm could control the final decision as in the present work with all of the information at hand. Alternatively, the algorithm may need to disentangle the biases from the private information of the human and use recommendations to only correct for the biases. 
In general, this is a difficult task requiring appropriate assumptions and bounds on judges' private information and utility functions \citep{rambachan2021identifying}.



Additionally, in either the static or strategic settings,
it would be interesting to consider either more complex, or more constrained, interactions between algorithms and decision makers. 
Empirical research has shown that the way recommendations are presented can affect people's reactions to the recommendations \citep{yeomans2019making}. 
If an algorithm could provide an explanation together with its recommendation (e.g., in the form of text or highlighting important features of an individual), how much would this impact the ability of the algorithm to persuade a judge with differing objectives?
On the other hand, it would be realistic to characterize the limitations to control that arise due to constraints on decision-aids. 
For example, the set of available recommendations may be discrete and pre-determined, in contrast to the cheap talk game where discrete recommendations emerge organically from the equilibrium. In a binary decision setting, the algorithm must learn from past binary decisions but only needs to control for a binary outcome.

Lastly, this paper does not address how a strategic judge detects the presence of manipulation, or the social welfare implications when the algorithm tries to control the final decision, under different assumptions about the response of the judge and his private information.
Such investigations may be fruitful areas for future work.





\bibliographystyle{plainnat}
\bibliography{bibliography}

\begin{thebibliography}{21}
\providecommand{\natexlab}[1]{#1}
\providecommand{\url}[1]{\texttt{#1}}
\expandafter\ifx\csname urlstyle\endcsname\relax
  \providecommand{\doi}[1]{doi: #1}\else
  \providecommand{\doi}{doi: \begingroup \urlstyle{rm}\Url}\fi

\bibitem[Camerer(1995)]{camerer1995individual}
Colin Camerer.
\newblock Individual decision making.
\newblock \emph{Handbook of experimental economics}, 1995.

\bibitem[Crawford and Sobel(1982)]{crawford1982strategic}
Vincent~P Crawford and Joel Sobel.
\newblock Strategic information transmission.
\newblock \emph{Econometrica: Journal of the Econometric Society}, pages
  1431--1451, 1982.

\bibitem[De-Arteaga et~al.(2020)De-Arteaga, Fogliato, and
  Chouldechova]{de2020case}
Maria De-Arteaga, Riccardo Fogliato, and Alexandra Chouldechova.
\newblock A case for humans-in-the-loop: Decisions in the presence of erroneous
  algorithmic scores.
\newblock In \emph{Proceedings of the 2020 CHI Conference on Human Factors in
  Computing Systems}, pages 1--12, 2020.

\bibitem[Donahue et~al.(2022)Donahue, Chouldechova, and
  Kenthapadi]{10.1145/3531146.3533221}
Kate Donahue, Alexandra Chouldechova, and Krishnaram Kenthapadi.
\newblock Human-algorithm collaboration: Achieving complementarity and avoiding
  unfairness.
\newblock In \emph{2022 ACM Conference on Fairness, Accountability, and
  Transparency}, FAccT '22, page 1639–1656, New York, NY, USA, 2022.
  Association for Computing Machinery.
\newblock ISBN 9781450393522.
\newblock \doi{10.1145/3531146.3533221}.
\newblock URL \url{https://doi.org/10.1145/3531146.3533221}.

\bibitem[Eren and Mocan(2018)]{eren2018emotional}
Ozkan Eren and Naci Mocan.
\newblock Emotional judges and unlucky juveniles.
\newblock \emph{American Economic Journal: Applied Economics}, 10\penalty0
  (3):\penalty0 171--205, 2018.

\bibitem[Gonzalez and Wu(1999)]{gonzalez1999shape}
Richard Gonzalez and George Wu.
\newblock On the shape of the probability weighting function.
\newblock \emph{Cognitive psychology}, 38\penalty0 (1):\penalty0 129--166,
  1999.

\bibitem[Heyes and Saberian(2019)]{heyes2019temperature}
Anthony Heyes and Soodeh Saberian.
\newblock Temperature and decisions: evidence from 207,000 court cases.
\newblock \emph{American Economic Journal: Applied Economics}, 11\penalty0
  (2):\penalty0 238--265, 2019.

\bibitem[Imai et~al.(2020)Imai, Jiang, Greiner, Halen, and
  Shin]{imai2020experimental}
Kosuke Imai, Zhichao Jiang, James Greiner, Ryan Halen, and Sooahn Shin.
\newblock Experimental evaluation of algorithm-assisted human decision-making:
  Application to pretrial public safety assessment.
\newblock \emph{arXiv preprint arXiv:2012.02845}, 2020.

\bibitem[Jacobs et~al.(2021)Jacobs, Pradier, McCoy~Jr, Perlis, Doshi-Velez, and
  Gajos]{jacobs2021machine}
Maia Jacobs, Melanie~F Pradier, Thomas~H McCoy~Jr, Roy~H Perlis, Finale
  Doshi-Velez, and Krzysztof~Z Gajos.
\newblock How machine-learning recommendations influence clinician treatment
  selections: the example of antidepressant selection.
\newblock \emph{Translational psychiatry}, 11\penalty0 (1):\penalty0 108, 2021.

\bibitem[Kahneman and Tversky(1979)]{kahneman_prospect_1979}
Daniel Kahneman and Amos Tversky.
\newblock Prospect theory: An analysis of decision under risk.
\newblock 47\penalty0 (2):\penalty0 263--291, 1979.
\newblock ISSN 0012-9682.
\newblock \doi{10.2307/1914185}.
\newblock URL \url{https://www.jstor.org/stable/1914185}.
\newblock Publisher: [Wiley, Econometric Society].

\bibitem[Kleinberg et~al.(2018)Kleinberg, Lakkaraju, Leskovec, Ludwig, and
  Mullainathan]{kleinberg2018human}
Jon Kleinberg, Himabindu Lakkaraju, Jure Leskovec, Jens Ludwig, and Sendhil
  Mullainathan.
\newblock Human decisions and machine predictions.
\newblock \emph{The quarterly journal of economics}, 133\penalty0 (1):\penalty0
  237--293, 2018.

\bibitem[Krantz and Parks(2002)]{krantz2002implicit}
Steven~George Krantz and Harold~R Parks.
\newblock \emph{The implicit function theorem: history, theory, and
  applications}.
\newblock Springer Science \& Business Media, 2002.

\bibitem[Ludwig and Mullainathan(2021)]{ludwig2021fragile}
Jens Ludwig and Sendhil Mullainathan.
\newblock Fragile algorithms and fallible decision-makers: lessons from the
  justice system.
\newblock \emph{Journal of Economic Perspectives}, 35\penalty0 (4):\penalty0
  71--96, 2021.

\bibitem[Mullainathan and Obermeyer(2022)]{mullainathan2022diagnosing}
Sendhil Mullainathan and Ziad Obermeyer.
\newblock Diagnosing physician error: A machine learning approach to low-value
  health care.
\newblock \emph{The Quarterly Journal of Economics}, 137\penalty0 (2):\penalty0
  679--727, 2022.

\bibitem[Rambachan(2021)]{rambachan2021identifying}
Ashesh Rambachan.
\newblock Identifying prediction mistakes in observational data.
\newblock Technical report, Working paper, 2021.

\bibitem[Rastogi et~al.(2022)Rastogi, Leqi, Holstein, and Heidari]{RLHH22}
Charvi Rastogi, Liu Leqi, Kenneth Holstein, and Hoda Heidari.
\newblock A unifying framework for combining complementary strengths of humans
  and {ML} toward better predictive decision-making.
\newblock \emph{CoRR}, abs/2204.10806, 2022.
\newblock \doi{10.48550/arXiv.2204.10806}.
\newblock URL \url{https://doi.org/10.48550/arXiv.2204.10806}.

\bibitem[Steury and Frank(1990)]{steury1990gender}
Ellen~Hochstedler Steury and Nancy Frank.
\newblock Gender bias and pretrial release: More pieces of the puzzle.
\newblock \emph{Journal of Criminal Justice}, 18\penalty0 (5):\penalty0
  417--432, 1990.

\bibitem[Stevenson and Doleac(2022)]{stevenson2022algorithmic}
Megan~T Stevenson and Jennifer~L Doleac.
\newblock Algorithmic risk assessment in the hands of humans.
\newblock \emph{Available at SSRN 3489440}, 2022.

\bibitem[Suresh and Guttag(2021)]{suresh2021framework}
Harini Suresh and John Guttag.
\newblock A framework for understanding sources of harm throughout the machine
  learning life cycle.
\newblock In \emph{Equity and access in algorithms, mechanisms, and
  optimization}, pages 1--9. 2021.

\bibitem[Tversky and Kahneman(1992)]{tversky1992advances}
Amos Tversky and Daniel Kahneman.
\newblock Advances in prospect theory: Cumulative representation of
  uncertainty.
\newblock \emph{Journal of Risk and uncertainty}, 5\penalty0 (4):\penalty0
  297--323, 1992.

\bibitem[Yeomans et~al.(2019)Yeomans, Shah, Mullainathan, and
  Kleinberg]{yeomans2019making}
Michael Yeomans, Anuj Shah, Sendhil Mullainathan, and Jon Kleinberg.
\newblock Making sense of recommendations.
\newblock \emph{Journal of Behavioral Decision Making}, 32\penalty0
  (4):\penalty0 403--414, 2019.

\end{thebibliography}

\newpage
\appendix

\section{General Identification Result}\label{app:identification}

\begin{theorem}\label{thm:identify_gen}
    Consider known nonlinear and differentiable functions $h_i:\R^{1+L}\to [0,1]$.
    Suppose that for $i\in\{1,\dots, B\}$ observations are generated according to
    \[y_i = h_i(\beta^\top z_i, \theta)\]
    for known variables $z_i$ and unknown parameters $\beta\in\mathcal B$ and $\theta\in\mathcal T$, where $\mathcal B$ and $\mathcal T$ are simply connected and compact subsets of $\mathbb{R}^D$ and $\mathbb{R}^L$ respectively.

    Denote by $h_i'$ the partial derivative of $h_i$ with respect to the first argument and $\nabla_\theta h_i$ the gradient with respect to the latter $L$ arguments.
    Define the matrices 
    \begin{align}
        \label{eq:identification-mx-appendix}
        Z = \begin{bmatrix}
         z_1^\top \\ \vdots \\   z_N^\top \end{bmatrix},
        \quad
        B_{\beta,\theta} = \begin{bmatrix}
        h_1'(\beta^\top z_1, \theta) && \\ &\ddots &\\  &&h_N'(\beta^\top z_N, \theta) \end{bmatrix},
        \quad
        \Theta_{\beta,\theta} = \begin{bmatrix}
            \nabla_\theta h_1(\beta^\top z_1, \theta)\\
            \vdots \\ 
            \nabla_\theta h_N( \beta^\top z_N, \theta)
        \end{bmatrix}
    \end{align}
    Then the parameters $\beta$ and $\theta$ can be uniquely identified from a dataset of $\{z_i, y_i\}_{i=1}^N$ if the following conditions hold for all $\beta\in\mathcal B$ and $\theta\in\mathcal T$:
    \begin{itemize}
        \item Rank condition: $B_{\beta,\theta}Z$ and $\Theta_{\beta,\theta}$ are full rank, i.e. $\rank(B_{\beta,\theta}Z) = D$ and $\rank(\Theta_{\beta,\theta}) = L$ 
        \item Independence condition: the column spaces of $B_{\beta,\theta}X$ and $\Theta_{\beta,\theta}$ are perpendicular, i.e. $\colspace(B_{\beta,\theta}X) \perp \colspace(\Theta_{\beta,\theta})$.
    \end{itemize}
\end{theorem}

\begin{lemma}\label{lem:localinv}
Define the observation map $F_N:\mathcal B\times \mathcal T \to \mathbb R^N$ as
\[F_N(\beta, \theta) = \begin{bmatrix}
h_1( \beta^\top z_1, \theta)\\
\vdots \\ 
h_N( \beta^\top z_N, \theta)
\end{bmatrix}\]
The Jacobian of $F_N$ is invertible at $\beta,\theta$ if and only if the rank and independence conditions hold.
\end{lemma}
\begin{proof}
Denote by $J$ the Jacobian of $F_N$. Then
\[J = \begin{bmatrix}\nabla h_1(\beta^\top z_1,\theta) \\ \vdots \\\nabla g_N(\beta^\top z_N,\theta)  \end{bmatrix}= \begin{bmatrix}B_{\beta,\theta}X & \Theta_{\beta,\theta}\end{bmatrix}\:.\]
The Jacobian $J$ is invertible if and only if the nullspace of $J$ contains only zero.

We first argue that the rank and independence conditions are sufficient. Suppose that $Jv=0$ for some $v$.
Letting $v=[v_1,v_2]$, this is equivalent to $BXv_1 + \Gamma v_2=0$.
Notice these terms are elements of $\colspace(BX)$ and $\colspace(\Gamma)$ respectively. 
By the independence condition, it must be that $BXv_1=0$ and $\Gamma v_2=0$.
By the rank condition and the rank-nullity theorem, it must be that $v_1=0$ and $v_2=0$.
Thus the rank and independence conditions imply that $J$ is invertible.

We now show that the rank and independence conditions are necessary.
If the independence condition does not hold, there is some nonzero $u$ 
such that $u=BXv_1=\Gamma v_2$.
Then $v=[v_1,-v_2]\neq 0$ is in the nullspace of $J$ so $J$ is not invertible.
If either $DX$ or $\Gamma$ is not full rank, then a nonzero element of their nullspace can be used to construct a nonzero element of the nullspace of then $J$.
This concludes the proof.
\end{proof}

\begin{proof}[Proof of Proposition~\ref{thm:identify_gen}]
Define $\mathcal F\subseteq \mathbb R^N$ as the image of the map $F_N$
defined in Lemma~\ref{lem:localinv}.
With some abuse of notation, we will now consider the function $F_N:\mathcal B \times \mathcal G\to \mathcal F$.
Identifiability of the parameters $\beta$ and $\gamma$ is equivalent to global invertibility of the function $F_N$.
We will use a Theorem due to Hadamard~\cite[Theorem 6.2.8]{krantz2002implicit} 
which states that
$F_N$ is globally invertible if it is proper, if the Jacobian never vanishes, and if $\mathcal F$ is simply connected.

$F_N$ is proper because each $h_i$ is proper and $Z$ is full rank.
Since $\mathcal F$ is the image of a simply connected space under a continuous mapping, it is also simply connected.
Finally, by Lemma~\ref{lem:localinv}, the Jacobian of $F_N$ is everywhere invertible under the rank and independence conditions.\end{proof}

\end{document}